\documentclass{article}

\usepackage[preprint]{neurips_2020}

\usepackage[utf8]{inputenc} 
\usepackage[T1]{fontenc}    
\usepackage{xcolor}
\definecolor{mydarkblue}{rgb}{0,0.08,0.45}
\usepackage[urlcolor  = mydarkblue]{hyperref}       
\usepackage{url}            
\usepackage{booktabs}       
\usepackage{amsfonts}       
\usepackage{nicefrac}       
\usepackage{microtype}      
\usepackage{mathrsfs}
\usepackage{bm}

\usepackage{amssymb,amsmath}
\usepackage{bbding}
\usepackage{amsthm}

\newtheorem{theorem}{Theorem}[section]

\newtheorem{lemma}[theorem]{Lemma}

\renewenvironment{proof}{\textbf{Proof.}}{\QED\bigskip}

\newcommand{\BEAS}{\begin{eqnarray*}}
\newcommand{\EEAS}{\end{eqnarray*}}
\newcommand{\BEA}{\begin{eqnarray}}
\newcommand{\EEA}{\end{eqnarray}}
\newcommand{\BEQ}{\begin{equation}}
\newcommand{\EEQ}{\end{equation}}
\newcommand{\BIT}{\begin{itemize}}
\newcommand{\EIT}{\end{itemize}}
\newcommand{\BNUM}{\begin{enumerate}}
\newcommand{\ENUM}{\end{enumerate}}
\newcommand{\BPM}{\begin{pmatrix}}
\newcommand{\EPM}{\end{pmatrix}}

\newcommand{\BA}{\begin{array}}
\newcommand{\EA}{\end{array}}

\newcommand{\ones}{\mathbf 1}

\newcommand{\reals}{{\mathbb R}}

\newcommand{\symm}{{\mbox{\bf S}}}  

\newcommand{\Tr}{\mathop{\bf Tr}}
\newcommand{\diag}{\mathop{\bf diag}}

\newcommand{\idm}{\mathbf{I}}

\newcommand{\QED}{~~\rule[-1pt]{6pt}{6pt}}

\providecommand{\cN}{{\mathcal N}} 
\providecommand{\cO}{{\mathcal O}} 
\providecommand{\tX}{\tilde{X}}

\providecommand{\0}{{\mathbf 0}}

\providecommand{\R}{{\mathbb R}} 

\providecommand{\s}{{\mathbf s}} 

\providecommand{\tildeD}{{\tilde{D}}}

\providecommand{\cove}{\Sigma} 

\providecommand{\abs}[1]{\left\lvert#1\right\rvert}
\providecommand{\call}[2]{\textsc{\small #1}\ensuremath{\left( #2 \right)}}

\definecolor{mydarkblue}{rgb}{0,0.08,0.6}
\hypersetup{
        colorlinks=true,
        citecolor=mydarkblue,
        linkcolor=black,
    }

\usepackage{algorithmic,algorithm}

\usepackage[off]{auto-pst-pdf} 

\title{FANOK: Knockoffs in Linear Time}

\author{Armin Askari\thanks{Equal Contribution}\\
UC Berkeley\\
\texttt{aaskari@berkeley.edu}
\And
Quentin Rebjock$^\ast$\\ 
EPFL, Inria\\
\texttt{quentin.rebjock@epfl.ch}\\
\And
Alexandre d'Aspremont\\
CNRS \& ENS Paris.\\
\texttt{aspremon@ens.fr}\\
\AND
Laurent El Ghaoui\\
UC Berkeley\\
\texttt{elghaoui@berkeley.edu}\\
}

\begin{document}    

\maketitle

\setcounter{footnote}{0} 

\begin{abstract}
We describe a series of algorithms that efficiently implement Gaussian {\em model-X} knockoffs to control the false discovery rate on large scale feature selection problems. Identifying the knockoff distribution requires solving a large scale semidefinite program for which we derive several efficient methods. One handles generic covariance matrices, has a complexity scaling as $\cO(p^3)$ where $p$ is the ambient dimension, while another assumes a rank $k$ factor model on the covariance matrix to reduce this complexity bound to $O(pk^2)$. We also derive efficient procedures to both estimate factor models and sample knockoff covariates with complexity linear in the dimension. We test our methods on problems with $p$ as large as 500,000. \footnote{A python implementation of our model can be found at \url{https://github.com/qrebjock/fanok}}
\end{abstract}

\section{Introduction}\label{s:intro}

Feature selection is a key preprocessing step in prediction tasks. Pruning out irrelevant variables both improves test performance by reducing noise and helps interpretation by focusing the prediction task on a short list of important variables. In many cases, the variable selection step is in fact more important than the prediction itself. The tradeoff between prediction performance and model size is typically very favorable. However, feature selection needs to select among an exponential number of hypotheses (the subset of selected variables) using a limited number of samples, and is thus naturally exposed to false discoveries. A lot of effort has been focused on controlling the false discovery rate (FDR) in feature selection, with notably \cite{Benj95} controlling FDR using $p$-values. These results work well in settings where $p$-values are readily available and has been extended, in part, to more sophisticated feature selection procedures in what is known  as post selection inference (see e.g. \citep{Berk13,Lee16}). This requires computing $p$-values after complex prediction tasks, which is far from trivial. 

A more flexible alternative is provided by the {\em knockoff framework} developed in \cite{Barb15,Cand18,Barb19}. In this setting, we first generate {\em knockoff} covariates whose distribution roughly matches that of the true covariates, except that knockoffs are designed to be conditionally independent of the response, and hence should never be selected by a feature selection procedure. This last fact helps in controlling the false discovery rate. The procedure in \cite{Cand18} shows how to design knockoffs in the Gaussian case and requires solving a semidefinite program (SDP) of dimension $p$ equal to the ambient dimension. While the knockoff framework does not explicitly control power, the SDP optimally decorrelates true covariates and their knockoff, which empirically improves power. The current package provided by the authors of \cite{Cand18} uses generic interior point methods (IPM), which scale roughly as $\cO(p^{4.5})$, which can be reduced to $\cO(p^{3.5})$ using problem structure \citep{Boyd03}. Feature selection is naturally a high dimensional problem, making generic IPM solvers ill suited for the task. Simple tricks produce simple feasible solutions to the knockoff SDP, but at the expense of a loss in power. Clustering the covariance matrix also allows \cite{Cand18} to solve much larger problems, but the limitations on maximum block size remains.


Here, we use problem structure to derive a block coordinate descent method and solve a barrier formulation as in e.g. \citep{dAsp06b,Wen09}. Iterations require low rank Cholesky updates which can be handled efficiently. This allows us to to produce a first algorithm which handles generic covariance matrices, and has a complexity scaling as $\cO(p^3)$ where $p$ is the ambient dimension. We then derive another method which assumes a rank $k$ factor model on the covariance matrix to reduce this complexity bound to $O(pk^2)$. This last method is potentially unstable in very particular scenarios, but we do not observe instabilities in practice.  We also derive efficient procedures to both estimate factor models and sample knockoff covariates with complexity linear in the dimension. We test our methods on problems with $p$ as large as 500,000.


\subsection{Notation} 
Let $[p] = \{1,\hdots,p\}$. Given $M \in \R^{p \times p}$ and two sets of indices $I, J \subseteq [p]$, $M_{I,J}$ denotes the $\abs{I} \times \abs{J}$ matrix obtained by keeping the $\abs{I}$ rows and $\abs{J}$ columns indexed by $I$ and $J$ respectively. For simplicity, an integer $j$ denotes the set $\left\{ j \right\}$, $j^c$ denotes the set $[p] \setminus \left\{ j \right\}$ and $:$ denotes either all rows or columns in the matrix subscript context. For example,
\begin{align*}
    M = \begin{bmatrix}
    1 & 2 & 3\\
    4 & 5 & 6\\
    7 & 8 & 9
\end{bmatrix} \;\; \Longrightarrow \;\; M_{1^c, 1^c} = \begin{bmatrix}
    5 & 6\\
    8 & 9
\end{bmatrix}, \;\; M_{1^c, 1} = \begin{bmatrix}
    4\\
    7
\end{bmatrix}, \;\; M_{1^c, :} = \begin{bmatrix}
    4 & 5 & 6\\
    7 & 8 & 9
\end{bmatrix}, \;\; M_{1,1} = 1
\end{align*}
For $s \in \mathbb{R}^p$, $D = \diag (s)$ denotes a $p\times p$ diagonal matrix with $D_{ii} = s_i$. For $M \in \mathbb{R}^{p \times p}$, $m = \diag (M)$ denotes a vector in $\mathbb{R}^p$ with $m_i = M_{ii}$. Unless otherwise stated, $x_j$ and $x_j^\prime$ denote the $j^\mathrm{th}$ column and row of a matrix $X$ respectively. $\symm_p$ denotes the set of $p \times p$ symmetric matrices.

\subsection{Primer on Knockoffs}
Given random covariates and a random response $(x, y) \in \R^{p} \times \R$, the knockoff framework of \cite{Barb15,Cand18,Barb19} seeks to control the false discovery rate in feature selection by constructing a new family of random variables $\tilde x \in \R^p$ called {\em knockoffs} which have a joint distribution comparable to their counterparts $x$ but are independent of the response $y$. As a result, these knockoff variables should not be selected by any reasonable feature selection procedure. The knockoff framework controls the FDR by keeping the features which are more strongly selected than their knockoff counterpart (which usually requires solving a LASSO-type problem; see Section 3.2 of~\cite{Cand18}).

More specifically, the {\em model-X knockoff} framework of~\cite{Cand18} formally defines knockoffs as a new family of random variables $\tilde x \in \mathbb{R}^p$ such that $\tilde x \perp \!\!\! \perp y \mid x$, and for any $S\subset[p]$, $[x;\tilde x]$ satisfies
\[
[x;\tilde x]_{\mathrm{swap}(S)} \overset{d}{=} [x;\tilde x]
\]
where $[x;\tilde x]_{\mathrm{swap}(S)}$ is obtained from $[x;\tilde x]$ by swapping the $j$th entries of $x$ and $\tilde x$ for all $j \in S$. In the Gaussian case where $x \sim \mathcal{N}(0,\Sigma)$ this invariance property means that $[x;\tilde x]$ is also Gaussian with covariance matrix given by 
\[
  \begin{bmatrix}
    \Sigma & \Sigma - \diag (s)\\
    \Sigma - \diag (s) & \Sigma
\end{bmatrix} \succeq 0 
\]
for some $s \in \R^p$ such that the matrix is positive semidefinite (PSD), i.e. such that $0 \preceq \diag (\s) \preceq 2 \Sigma.$

Without loss of generality, we assume that $x$ is zero mean and that $\Sigma$ is a correlation matrix throughout. Given an observation $x$, Gaussian knockoffs are then sampled from the conditional distribution
$\tilde x \mid x \sim \cN\left( \mu,\, \Omega \right)$ such that 
\begin{align}\label{eq:conditional_gaussian_knockoffs}
\begin{split}
\mu &= x - \diag(s)\Sigma^{-1}x\\
\Omega &= \diag (s)\left( 2\idm_p - \Sigma^{-1}\diag (s) \right).
\end{split}
\end{align}
For the remainder of the paper, let $X \in \mathbb{R}^{p \times n}$ denote the scaled data matrix for the response vector $y \in \mathbb{R}^n$. After we sample all our knockoffs and aggregate them into the knockoff matrix $\tilde X \in \mathbb{R}^{p \times n}$, we compute a feature statistic in order to do feature selection. Intuitively, we want to construct knockoffs that are not ``too similar" to the original features (i.e. with low $\langle x_i^\prime, \tilde x_i^\prime \rangle = 1 - s_i$). To do so, we maximize the entries of $s$, solving the following SDP 
\BEQ\label{eq:sdp}
\BA{ll}
\mbox{maximize} & \ones^\top s\\
\mbox{subject to} & \diag (s) \preceq 2\Sigma\\
                &  0 \leq s \leq 1
\EA\EEQ
In this paper, we are concerned with solving \eqref{eq:sdp} as efficiently as possible. 


\section{Solving for Second Order Knockoffs}\label{s:sdp}
Solving the semidefinite program in \eqref{eq:sdp} using generic interior point methods \citep{Nest94,Helm96,Boyd03} has complexity $\mathcal{O}(p^{4.5})$ or $\mathcal{O}(p^{3.5})$ exploiting structure, which precludes their use on large-scale examples. In what follows, we will describe a coordinate ascent method that better exploits the structure of the problem. Each barrier problem has complexity $\cO(p^3)$, but when the covariance matrix $\Sigma$ is assumed to have a {\em diagonal plus low-rank} (aka factor model) structure, this complexity can be reduced to $\mathcal{O}(pk^2)$ where $k \ll p$. For simplicity, we will assume in this section that $\Sigma \succ 0$ which means in particular that $n \geq p$.
For the regime when $n < p$, see Section \ref{s:covest} on how adapt our method.

\subsection{A Basic Coordinate Ascent Algorithm}\label{ss:basic}
Here, as in \cite{Bane05,wen2012block}, we exploit the fact that the feasible set of program~\eqref{eq:sdp} has a product structure amenable to block coordinate ascent to derive an efficient algorithm for maximizing a barrier formulation of \eqref{eq:sdp} written
\BEQ\label{eq:sdp_log_barrier}
\BA{ll}
\mbox{maximize} & \ones^\top s + \lambda ~ \log\det\big( 2\Sigma - \diag (s) \big)\\
\mbox{subject to} & 0 \leq s \leq 1
\EA\EEQ
in the variable $s \in \R^p$, where $\lambda > 0$ is a barrier parameter. Note that the dual of \eqref{s:sdp} writes
\[\BA{ll}
\mbox{minimize} & 2\Tr (\Lambda \Sigma) + \ones^\top \eta\\
\mbox{subject to} & \diag (\Lambda)  + \eta \geq 1\\
& \Lambda \succeq 0, \; \eta \geq 0
\EA\]
and could be solved by adapting the block-coordinate method as in \cite{wen2012block}. Here however, we are focused on getting a solution to the primal problem in~\eqref{eq:sdp}, hence we focus on a block coordinate algorithm for solving~\eqref{eq:sdp_log_barrier}. We first recall the following key fact.

\begin{lemma}\label{lemma:det_matrix}
For any symmetric, invertible matrix $M \in \symm_p$ and any $j \in [p]$,
\begin{equation*}
    \det\left( M \right) = \big( M_{j,\, j} - M_{j^c,\, j}^\top M_{j^c,\, j^c}^{-1} M_{j^c,\, j} \big)
        \cdot\det\left( M_{j^c,\, j^c} \right)
\end{equation*}
\end{lemma}

On the barrier problem~\eqref{eq:sdp_log_barrier}, Lemma~\ref{lemma:det_matrix} yields
\[
    \log\det\big( 2\Sigma - \diag (s) \big) =
        \log\big( 2\Sigma_{j,\, j} - s_j - 4\Sigma_{j^c,\, j}^\top Q_j^{-1} \Sigma_{j^c,\, j} \big)
            + \log\det\left( Q_j \right)
\]
where $Q_j = 2\Sigma_{j^c,\, j^c} - \diag(s_{j^c})$ does not depend on $s_j$. Using this decomposition, maximizing over $s_j$ in~\eqref{eq:sdp_log_barrier} and leaving all other entries fixed, the first order optimality condition gives
\begin{equation}\label{eq:opt_cond}
    s_j^\star := \min \big(1,\max\left(
        2\Sigma_{j,\, j} - 4\Sigma_{j^c, j}^\top Q_j^{-1} \Sigma_{j^c, j} - \lambda, 0
    \right)\big)
\end{equation}
Applying this result iteratively yields the block coordinate ascent method detailed in Algorithm~\ref{alg:coordinate_ascent_log_barrier}.

\begin{algorithm}[t] 
    \caption{Coordinate ascent with log-barrier}\label{alg:coordinate_ascent_log_barrier}
    \begin{algorithmic}[1]
        \STATE \textbf{Input:} A covariance matrix $\Sigma\in\symm_p$, barrier coefficient $\lambda > 0$,
        decay $\mu < 1$, $s^{(0)} = 0 \in\reals^p$.
        \STATE Set $s = s^{(0)}$.
        \REPEAT
        \FOR{$j = 1,\,\dots, p$}
        \STATE Form $Q_j = 2\Sigma_{j^c,\, j^c} - \diag s_{j^c}$
        \STATE Compute $s_j = \min\big(1,\max\big( 2\Sigma_{j,\, j} - 4\Sigma_{j^c,\, j}^\top Q_j^{-1} \Sigma_{j^c,\, j} - \lambda,\, 0 \big)\big)$\label{alg:line:q_inv}
        \ENDFOR
        \STATE $\lambda = \mu \lambda$
        \UNTIL{stopping criteria}
        \STATE \textbf{Output:} A solution $s$.
    \end{algorithmic}
\end{algorithm}

\subsubsection{Iteration Complexity}
In Algorithm~\ref{alg:coordinate_ascent_log_barrier}, the bottleneck is the inversion of the matrix $Q_j \in \symm^+_{p-1}$ in line~\ref{alg:line:q_inv} which is $\cO (p^3)$, making the total complexity of Algorithm~\ref{alg:coordinate_ascent_log_barrier} $\cO (n_\text{iters}  p^4)$. We can however reduce the cost of Algorithm~\ref{alg:coordinate_ascent_log_barrier} to $\cO(n_\text{iters} p^3)$ by carefully updating $Q_j^{-1}$ between subsequent coordinates.

\begin{lemma} \label{lemma:low_rank_update}
Let $s\geq 0$ and $A = 2\Sigma - \diag (s)$. Then, for any $j \in [p]$, $Q_j^{-1}$ can be computed as the inverse of a rank-3 update on $A$.
\end{lemma}
\begin{proof}
Up to a permutation, we can assume without loss of generality that $j=1$. We can write
\[
\BPM
1 & 0\\
0 & Q_j^{-1}
\EPM
= \left(
A+e_je_j^T(1+2\Sigma_{jj})- 2 e_j\Sigma_j^T - 2 \Sigma_j e_j^T
\right)^{-1}
\]
where $e_j$ and $\Sigma_j$ is the $j^\mathrm{th}$ Euclidean basis vector and column of $\Sigma$ respectively.
\end{proof}

Using the Sherman-Woodbury-Morrisson (SWM) formula \citep{Golu90}
\BEQ\label{eq:swm}
(A+UV^T)^{-1}=A^{-1}-A^{-1}U(\idm+V^TA^{-1}U)^{-1}V^TA^{-1}.
\EEQ
updating $Q_j^{-1}$ has complexity $\cO(p^2)$. Note that $A$ enjoys a rank-1 modification when a coordinate of $s$ is updated. After the initial inversion of $\Sigma$, each update of $Q_j$ thus becomes an $\cO(p^2)$ operation and looping over all coordinates gives us a time complexity of $\cO(n_{\text{iters}} p^3)$. 

\subsubsection{Stable Updates}\label{sss:stable}
Despite this improvement in complexity, the biggest practical problem with the aforementioned scheme is the numerical instabilities present using the SWM formula~\citep{yip1986SM_unstable}. In order to circumvent this issue, we propose Algorithm \ref{alg:coordinate_ascent_log_barrier2} (see Section~\ref{ss:algs}) which is a modification of Algorithm  \ref{alg:coordinate_ascent_log_barrier} that uses Cholesky decompositions instead of matrix inversions.
The key step in  Algorithm \ref{alg:coordinate_ascent_log_barrier2} is a rank one update of $A = 2\Sigma - \diag (s)$ after updating a coordinate of $s$. Hence, given $A = LL^\top$, we can perform stable, rank one Cholesky updates on $A$ in $\cO(p^2)$ steps and solve a triangular system directly instead of inverting a matrix (see Section~\ref{s:stable_updates_cont} for more details). Hence, Algorithm \ref{alg:coordinate_ascent_log_barrier2} has the same worst-case complexity as Algorithm \ref{alg:coordinate_ascent_log_barrier}, but is both faster and more stable in practice. Despite this computational improvement, the complexity $\cO(n_{\text{iters}} p^3)$ is still prohibitive for large $p$. To make coordinate ascent scale, we assume in what follows that $\Sigma$ has a low-rank factor model structure (see Section~\ref{s:covest}) and adapt the method.

\subsection{Coordinate Ascent under Factor Model}\label{ss:sdp_low_rank}
The complexity of Algorithm~\ref{alg:coordinate_ascent_log_barrier} can be drastically reduced, from $\cO(n_{\text{iters}} p^3)$
to $\cO\left( n_\text{iters}  p  k^2 \right)$ assuming a low-rank factor model on $\Sigma$:
\BEQ\label{eq:factor-model}
\Sigma = D + UU^\top
\EEQ
where $D \succeq 0$ is a diagonal matrix, and $U \in \mathbb{R}^{p \times k}$ where $k \ll p$ (see Section~\ref{s:covest} for details on how efficiently estimate such a model). Under this assumption, for a given $j \in [p]$, using~\eqref{eq:swm}, we have
\begin{align}\label{eq:woodbury}
    Q_j^{-1} &= \left( 2D_{j^c,j^c} - \diag (s_{j^c}) + 2 U_{j^c,:} U_{j^c,:}^\top \right)^{-1} \nonumber \\
    &= \left( \tildeD_{j} + 2 U_{j^c,:} U_{j^c,:}^\top \right)^{-1} \nonumber \\
    &= \tildeD_{j}^{-1}  - 2 \tildeD_{j}^{-1}U_{j^c,:} \left( I_k + 2 U_{j^c,:}^\top \tildeD_{j}^{-1} U_{j^c,:} \right)^{-1} U_{j^c,:}^\top \tildeD_{j}^{-1}
\end{align}
where $\tildeD_j = 2D_{j^c,j^c} - \diag (s_{j^c})$. The computational gain comes from inverting a $k \times k$ matrix and diagonal matrix $\tildeD_j$ in \eqref{eq:woodbury} as opposed to a $(p-1) \times (p-1)$ matrix. Recall that at each iteration~$j$, only the $j$th coordinate of $s$ is updated with
\begin{equation*}
    s_j \leftarrow \min(1,\max(\alpha^\star,\, 0))
    \quad\text{where}\quad
    \alpha^\star = 2\cove_{j,\, j} - 4\cove_{j^c,\, j}^\top Q_j^{-1} \cove_{j^c,\, j} - \lambda
\end{equation*}
Using \eqref{eq:woodbury}, and the fact that under the factor model assumption $\Sigma_{j^c,j} = U_{j^c,:} U_{j,:}^\top$, we have
\BEQ\label{eq:decomp}
    \alpha^\star = \underbrace{2 \Sigma_{j,j} - 4 U_{j,:} M_j U_{j,:}^\top - \lambda}_\text{$(*)$} ~ + ~ \underbrace{8 U_{j,:} M_j (I_k + 2M_j)^{-1} M_j U_{j,:}^\top}_\text{$(**)$}
\EEQ
where $M_j = U_{j^c,:}^\top \tildeD_j^{-1} U_{j^c,:} \in \mathbb{R}^{k \times k}$. Forming $M_j$ directly costs $\cO(p k^2)$ but we can take advantage of the structure of $M_j$ to compute $(*)$ and $(**)$ efficiently (see Section~\ref{ss:ascent_factor_model_cont} for further details). One nuance to using the SWM formula in this way is the fact that $\tilde{D}_j$ can be nearly singular. In theory, this would preclude solving the SDP to arbitrary accuracy. In practice, this does not seem to be problem as numerical instabilities rarely occur (see Section \ref{s:numres}).

\section{Sampling Knockoffs}\label{s:sampling}
In this section, we detail how to generate the knockoff matrix $\tX \in \R^{p \times n}$
once an optimal solution $s$ to the semidefinite program~\eqref{eq:sdp} has been found. Each column $\tilde{x}_i$ is sampled according to the Gaussian conditional distribution in~\eqref{eq:conditional_gaussian_knockoffs}. This means sampling $\tilde{x}_i \mid x_i \sim \cN\left( \mu_i,\Omega \right)$ such that 
\begin{equation*}
    \mu_i = x_i - \diag(s) \Sigma^{-1} x_i
    \quad\mbox{and}
    \quad
    \Omega = 2\diag (s) - \diag (s) \Sigma^{-1} \diag(s).
\end{equation*}
Naively sampling from $\mathcal{N}(\mu_i,\Omega)$ via $\tilde{x}_i = \mu_i + L v$ where $v \sim \cN(0,\idm_p)$ and where $L$ satisfies $LL^\top = \Omega$ has complexity $\cO(p^3)$ (the cost associated with the Cholesky decomposition).

Suppose now that $\Sigma$ has a factor model structure as in~\eqref{eq:factor-model}; that is, $\Sigma = D + UU^\top$ where $D \succ 0$ is a diagonal matrix and $U \in \mathbb{R}^{p \times k}$ (with $k \ll p$).
We show how to factorize $\Omega$ and sample the knockoff matrix $\tX$ in $\cO\left( n  p  k^2 \right)$ steps using $\cO\left(p (n + k)\right)$ memory.
Using the factor model assumption and the SWM formula~\eqref{eq:swm}, we have
\begin{align*}
    \cove^{-1} &= D^{-1} - D^{-1} U N N^\top U^\top D^{-1}
\end{align*}
where $N \in \R^{k \times k}$ is the Cholesky factorization of
$(\idm_k + U^\top D^{-1} U)^{-1}$. Setting $S = \diag (s)$ gives
\begin{align*}
    \Omega &= 2S - S \cove^{-1} S = C + ZZ^\top
\end{align*}
where $C = 2S - SD^{-1}S$ is diagonal (but not necessarily psd) and $Z = S D^{-1} U N \in \R^{p \times k}$ is low-rank.
Forming $C$ and $Z$ takes at most $\cO\left(pk^2\right)$ operations and $\cO\left(pk\right)$ memory.
Notice that the mean $\mu_i$ is easily computed in $\cO\left( n p k^2 \right)$ operations
and without additional memory as follows
\begin{equation*}
    \mu_i = x_i - \cove^{-1} S x_i
    = x_i - D^{-1} S x_i + D^{-1} U N Z^\top x_i
\end{equation*}
For this reason, the problem reduces to sampling from $\cN\left(0,\Omega\right)$ efficiently.
To do so, we adopt the $L \Delta L^\top$ factorization procedure presented by~\cite{smola2004ldl},
which means decomposing $\Omega$ in the following way
\begin{equation}\label{eq:omega_ldlt}
    \Omega = C + Z Z^\top = L\left( Z,\, B \right)\; \Delta \; L\left( Z,\, B \right)^\top
\end{equation}
where $L\left( Z,\, B \right)$ (denoted $L$ in the sequel) has the following structure
\begin{align*}
    L\left( Z,\, B \right) &= \begin{pmatrix}
        1 \\
        \langle z_2', b_1' \rangle & 1 \\
        \vdots & & \ddots \\
        \langle z_p', b_1' \rangle & \dots & \langle z_p', b_{p - 1}' \rangle & 1
    \end{pmatrix}
\end{align*}
Here $z_i',b_i'$ denote the $i^\mathrm{th}$ row of $Z$ and $B$ respectively,
with $B \in \R^{p \times k}$, and $\Delta \in \symm_p^+$ is diagonal. \cite{smola2004ldl} detail how to construct $B$ and $\Delta$ (see Algorithm~\ref{alg:form_B} in  Appendix~\ref{appendix_alg}).

With a sample $v$ from $\cN\left(0,\idm_p\right)$, $\tilde{x}_i$ can be computed by setting $\tilde{x}_i = \mu_i + L \sqrt{\Delta}~v$.
The advantages of using the $L \Delta L^\top$ decomposition are that
\textit{(i)} we do not require $C \succeq 0$ and
\textit{(ii)} we never have to store the full matrices $L$ or $B$ to compute the product $L \sqrt{\Delta}~v$.
By virtue of the specific structure of $L$, its multiplication by a vector can be done in only $\cO(pk)$ operations (see Algorithm~\ref{alg:sampling_cholesky} in Appendix~\ref{appendix_alg}).
As $\Delta$ is diagonal, $\tilde{x}_i$ can be computed in only $\cO\left( p k \right)$ steps.
In practice, we derive an iterative procedure that never stores $B$ nor $L$ in memory to compute $\tilde{x}_i$.
Instead, their rows are computed on the fly and requires only $\cO(p + k^2)$ memory (see Algorithm~\ref{alg:merged_sampling} in Appendix~\ref{appendix_alg}). Finally, $n$ columns $\tilde{x}_i$ have to be sampled to form $\tilde{X} \in \R^{p \times n}$,
which requires $\cO\left( n p k^2 \right)$ steps and $\cO\left(p (n + k) \right)$ memory.

\section{Numerical Results}\label{s:numres}
All experiments utilized a standard workstation. For the plots below, all error bars represent one standard deviation.
Unless referring to our algorithms, all other functions used were from \texttt{Scikit-Learn}~\citep{scikit-learn}. For more details on experimental set up, see Appendix \ref{appendix_exp}.

\subsection{Benchmarks}
We first generate random covariance matrices and compare CPU time and quality of solutions in solving \eqref{eq:sdp} using \texttt{SCS} (a first order method) and \texttt{CVXOPT} (an IPM) interfaced with \texttt{cvxpy} \citep{o2016conic, andersen2011interior, diamond2016cvxpy} and solving \eqref{eq:sdp_log_barrier} using coordinate ascent. We set $\Sigma = D + V \Lambda V^\top$ where $D \in \mathbb{R}^{p \times p}, V \in \mathbb{R}^{p \times k}, \Lambda \in \mathbb{R}^{k \times k}$ where $D = 10^{-3} I_{p}, \;\; \Lambda_{ii}\sim U[0,1],\;\; V_{ij} \sim \mathcal{N}(0,1)$, and $k = \lceil 0.05p \rceil$. Figure \ref{fig:speed} shows the results of the experiment for increasing $p$ and the optimality of the generated solution (see  Figure \ref{fig:feas} in Section \ref{ss:benchmark_details} for the feasibility of the solution generated by coordinate ascent against a baseline).

\begin{figure}[h] 
  \centering
  \begin{minipage}[b]{0.49\textwidth}
    \includegraphics[width=\textwidth]{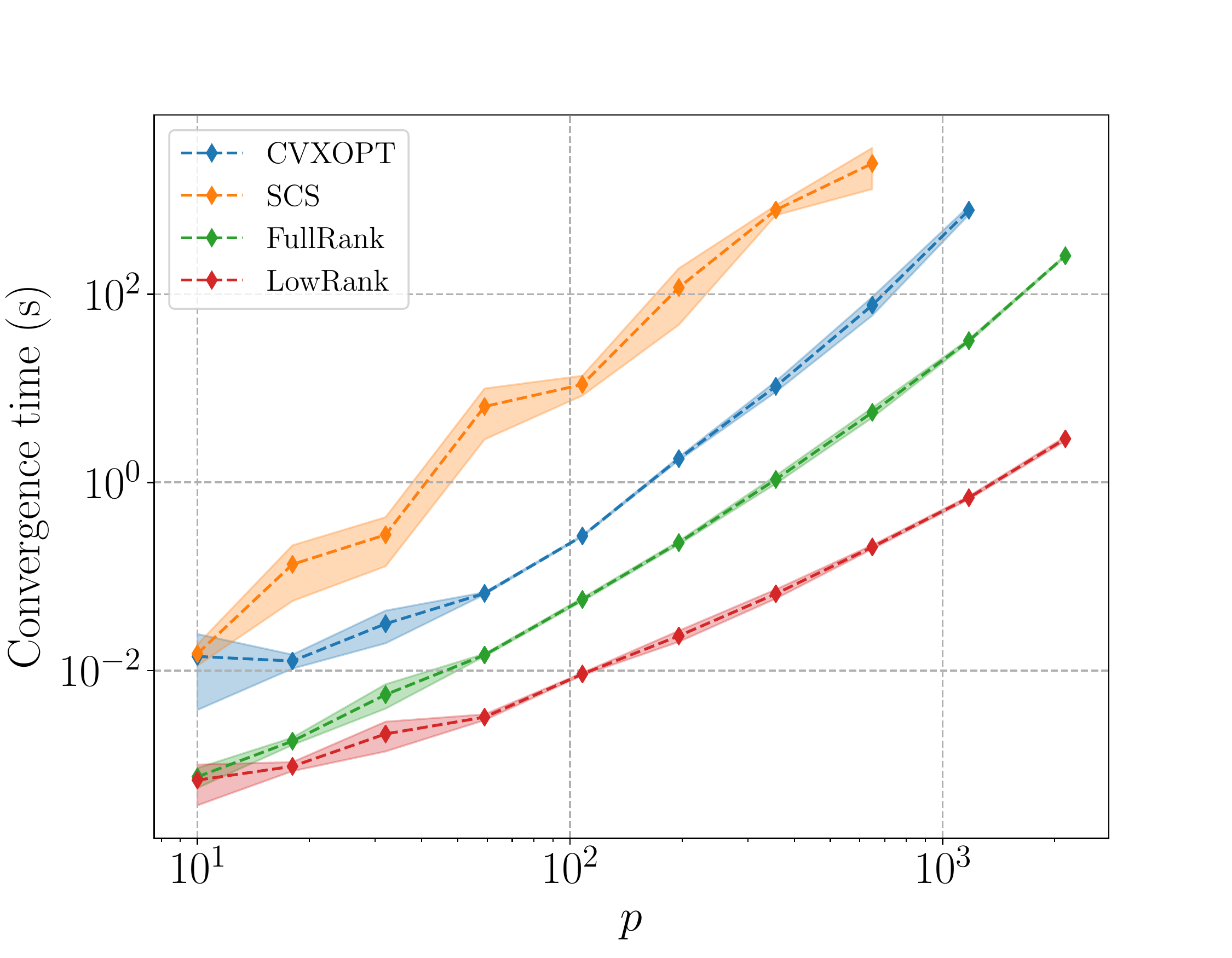}
  \end{minipage}
  \hfill
    \begin{minipage}[b]{0.49\textwidth}
    \includegraphics[width=\textwidth]{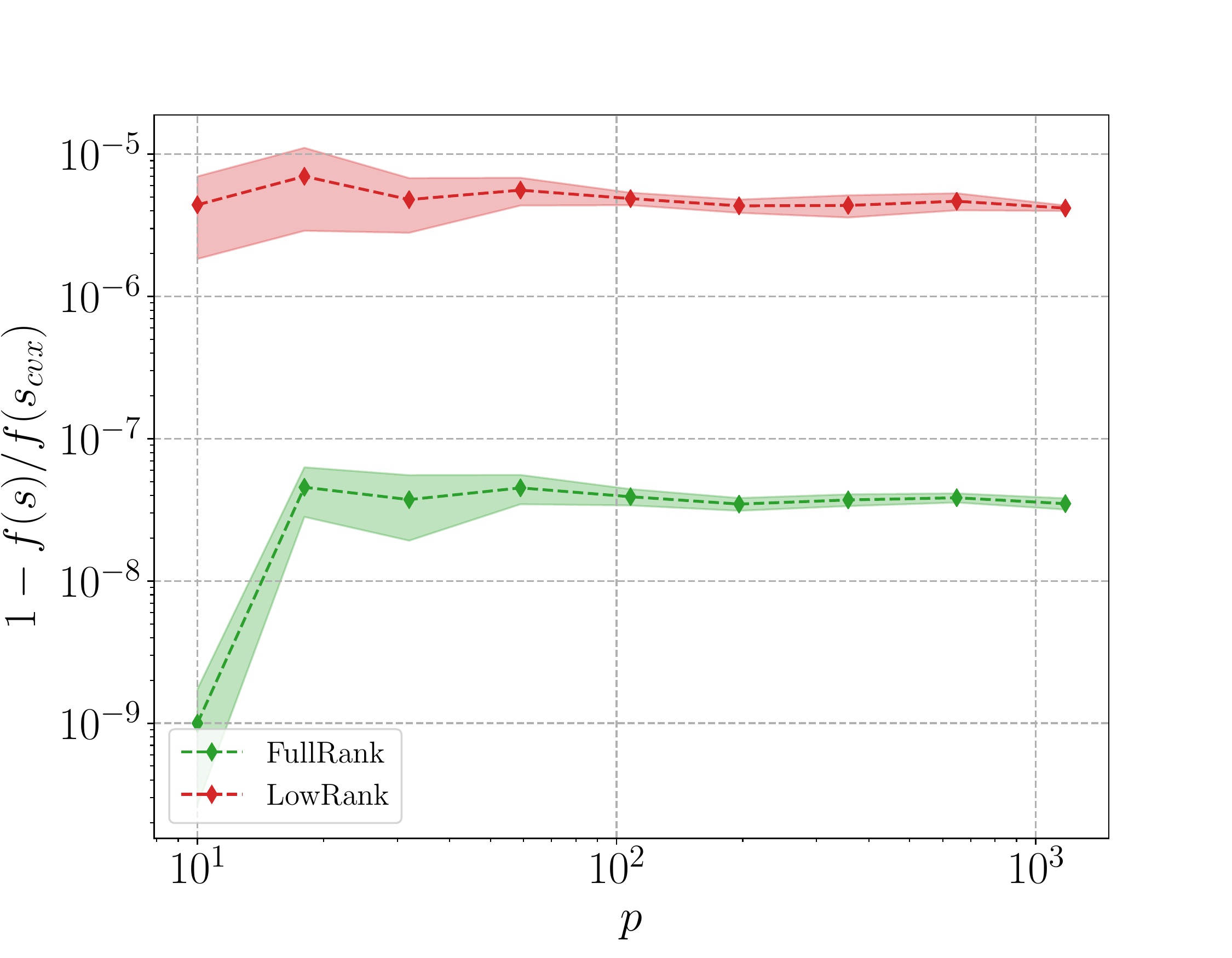}
  \end{minipage}
  \caption{(\textit{Left}) Convergence time versus dimension for solving \eqref{eq:sdp} using \texttt{CVXOPT} and \texttt{SCS} in \texttt{cvxpy} and using Algorithms \ref{alg:coordinate_ascent_log_barrier2} and \ref{alg:low_rank_coordinate_ascent} to solve \eqref{eq:sdp_log_barrier}. 
  (\textit{Right}) Objective values reached by the full and low rank algorithms relative to that generated using IPMs. Here $f(s) = \ones^\top s$.}\label{fig:speed}
\end{figure}

In Figure \ref{fig:speed}, coordinate ascent provides substantial computational gains compared to using \texttt{SCS} or \texttt{CVXOPT}. Solving the full rank model is consistently one (resp. two) orders of magnitude faster than \texttt{CVXOPT} (resp. \texttt{SCS}) and the low rank model for $p=500$ is four orders of magnitude faster than \texttt{SCS}. The slopes also indicate that for larger $p$, \texttt{SCS} and \texttt{CVXOPT} become prohibitively slow while the low rank model can comfortably handle $p \sim 10^5$. The right panel in Figure \ref{fig:speed} shows that the solution computed by our solver is indeed close to the \texttt{CVXOPT} solution (\texttt{SCS} produced infeasible solutions, see Section \ref{ss:benchmark_details}).

\subsection{Complexity}

We now check empirically the complexity bounds of Algorithm~\ref{alg:low_rank_coordinate_ascent} derived in Section~\ref{ss:sdp_low_rank} (under the factor model assumption). We focus on the time spent per cycle of the for loop in Algorithm \ref{alg:low_rank_coordinate_ascent}. We run two sets of experiments: one where we fix $k$ and increase $p$ and another where we fix $p$ and increase $k$. For both experiments, we generate covariance matrices as above. The results are plotted in Figure \ref{fig:comp}. This shows a favorable linear rate when $k \in [10^1, 10^2]$ and the theoretically derived quadratic rate when $k \in [10^2,10^3]$.\\

We also benchmark the complexity of Algorithm~\ref{alg:merged_sampling} to sample from $\cN(0, \Omega)$ when $\Omega = C + ZZ^\top$ where $C \in \R^{p \times p}$ is diagonal and PSD and $Z \in \R^{p \times k}$. In Figure~\ref{fig:comp} we compare it to the classical approach of computing the Cholesky factorization of $\Omega$ (note we use a plain python implementation for our algorithm). As seen in Figure~\ref{fig:comp}, Algorithm \ref{alg:merged_sampling} enjoys a linear dependence on $p$, a favorable (sub)linear rate when $k \in [1,10^2]$ and a quadratic dependence on $k$ when $k \in [10^2,10^3]$.
\begin{figure}[h] 
  \centering
    \begin{minipage}[b]{0.49\textwidth}
    \includegraphics[width=\textwidth]{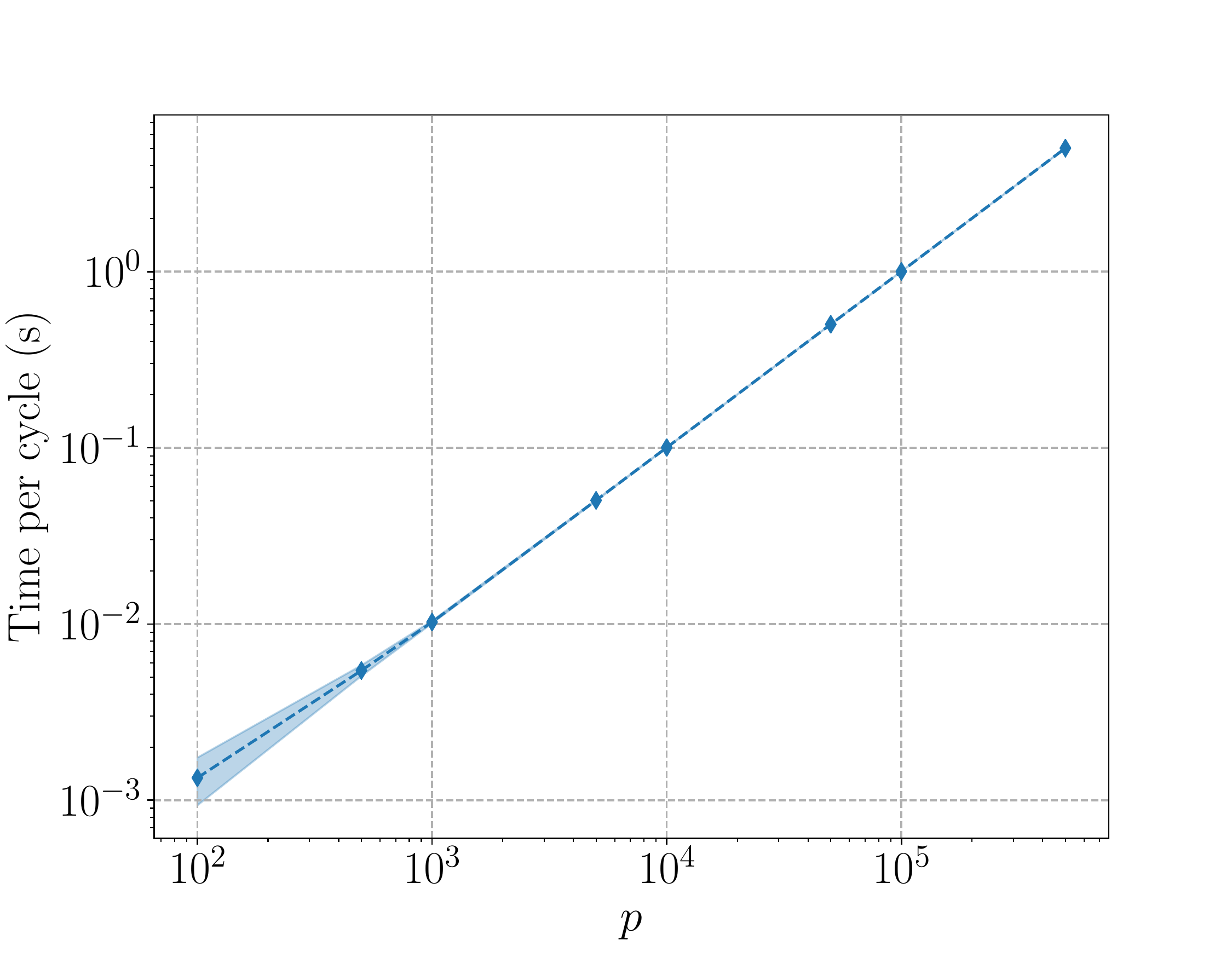}
  \end{minipage}
  \hfill
  \begin{minipage}[b]{0.49\textwidth}
    \includegraphics[width=\textwidth]{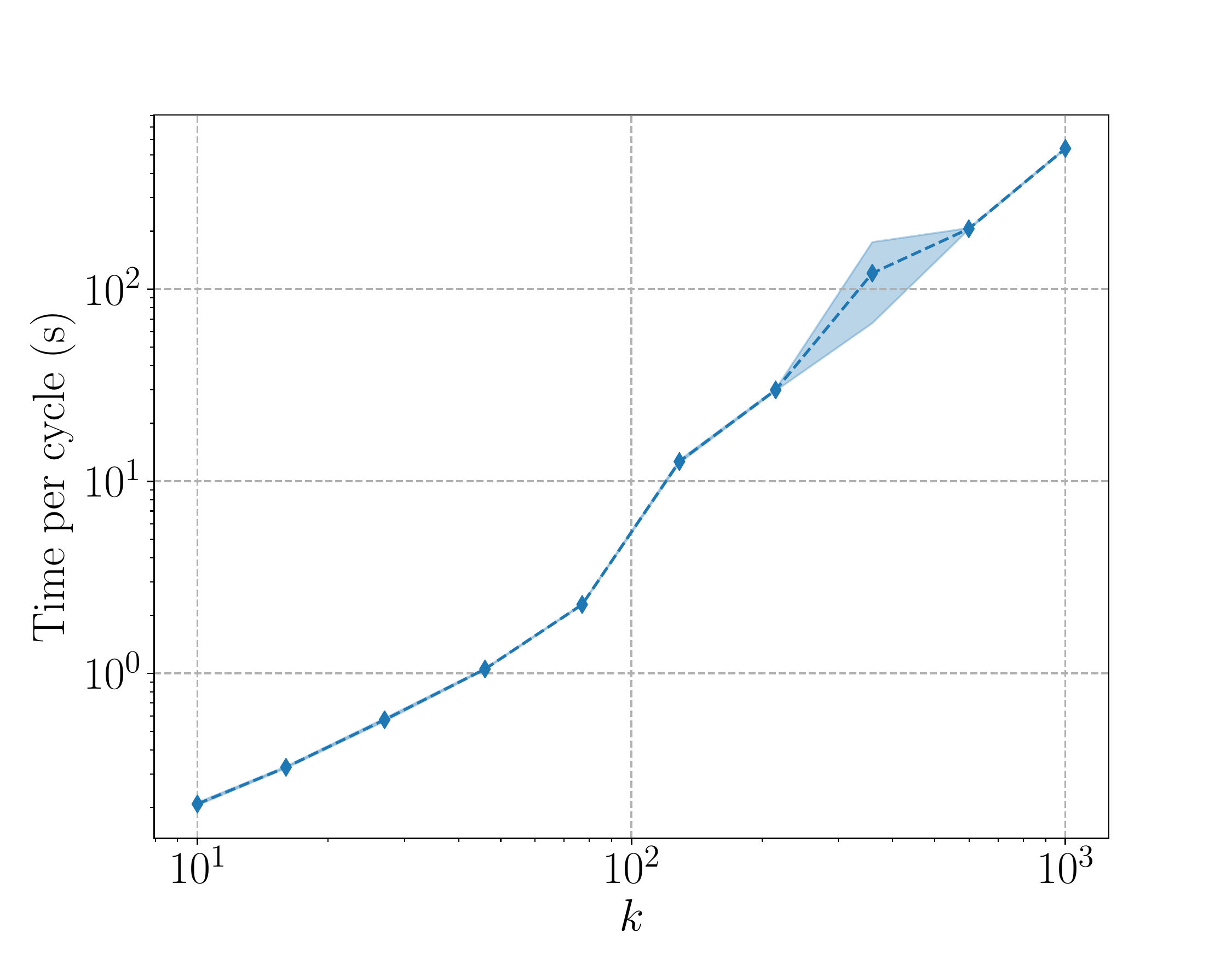}
  \end{minipage}
  \begin{minipage}[b]{0.49\textwidth}
    \includegraphics[width=\textwidth]{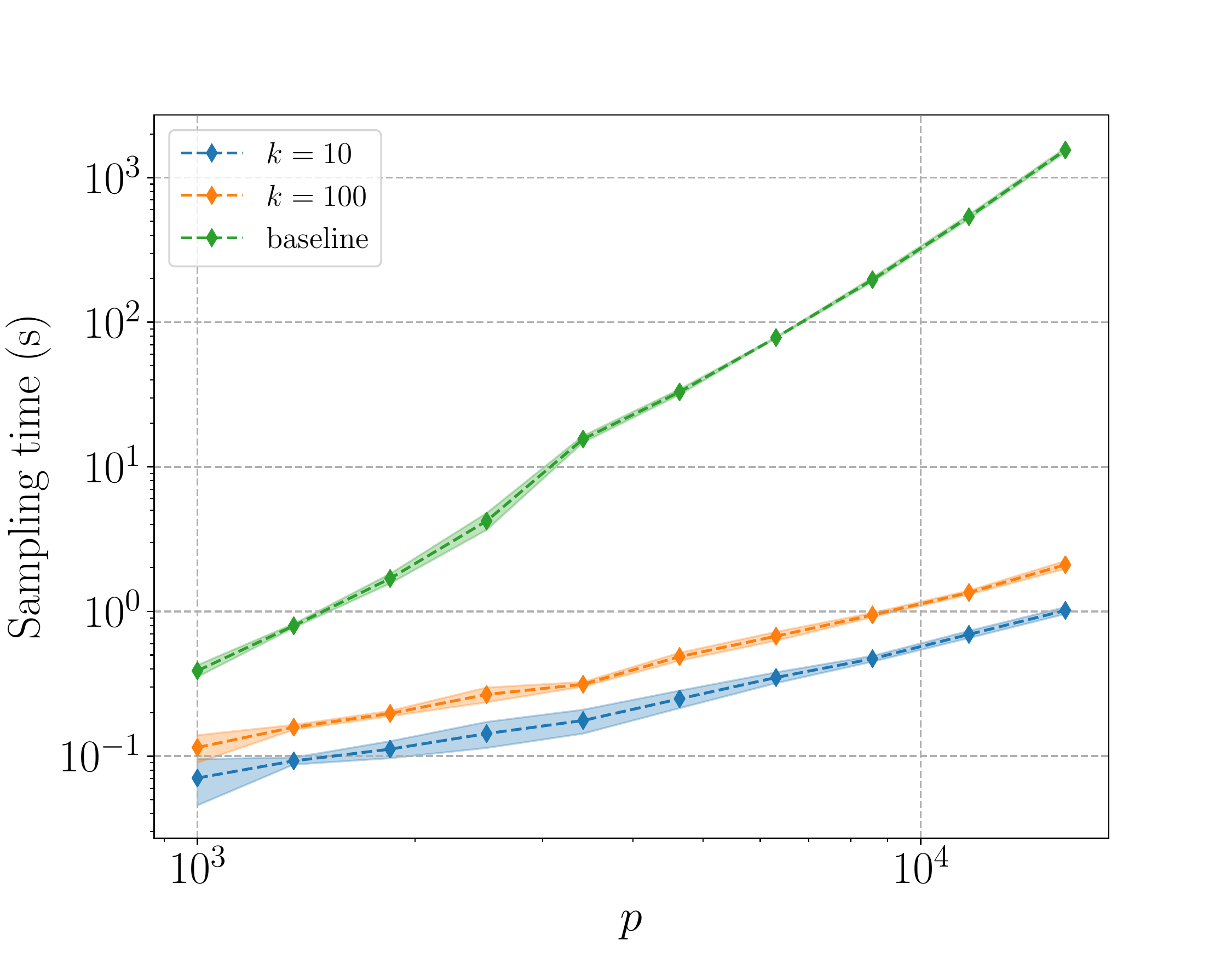}
  \end{minipage}
  \hfill
    \begin{minipage}[b]{0.49\textwidth}
    \includegraphics[width=\textwidth]{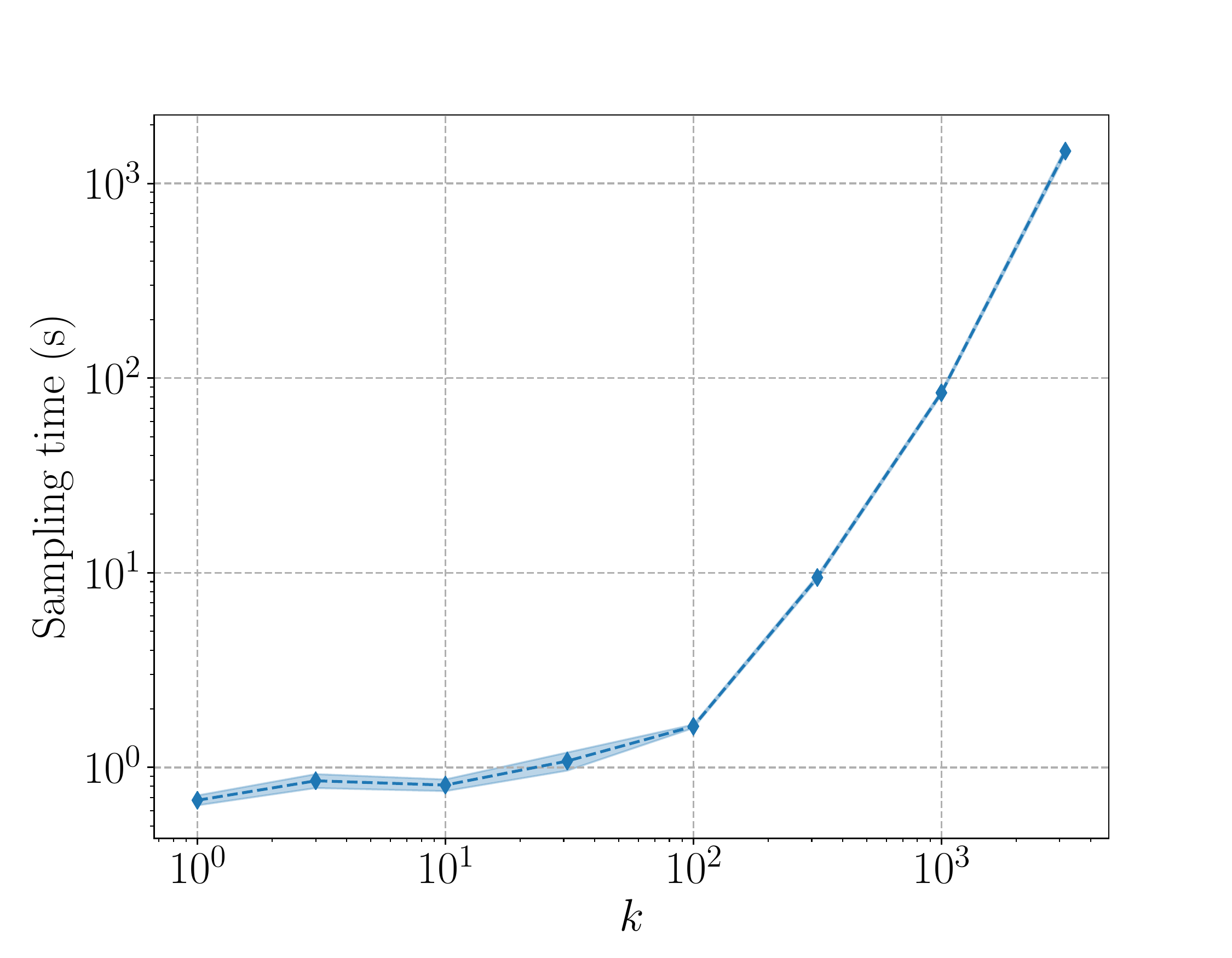}
  \end{minipage}
  \caption{ (\textit{Top Left})  Time per cycle versus dimension $p$ with $k=25$. This shows a linear dependence on $p$. (\textit{Top Right}) Time per cycle versus $k$ with $p=50,000$. \textit{(Bottom Left)} Sampling time versus dimension. \textit{(Bottom Right)} Sampling time versus rank for $p=25,000$.
  }\label{fig:comp}
  \vspace{-3mm}
\end{figure}

\subsection{FDR Control on Synthetic Data}
We now compare FDR control and power using different methods of solving \eqref{eq:sdp} at scale. For computational reasons, the two main current methods for constructing knockoffs in high dimension either use an equicorrelated (Equi) construction or an approximate semidefinite program (ASDP) construction (for more details see \citep[Section 3.4.2]{Cand18}). The Equi and ASDP constructions are approximations to the solution of \eqref{eq:sdp} and in this experiment, we compare the quality of knockoffs (measured via false discovery rate and power) generated using the above methods with the knockoffs generated via coordinate ascent, in the full rank and factor model settings.

We run a similar experiment to that in Figure 5 of \citep{Cand18}. We generate $\Sigma$ with $\Sigma = D + VV^\top$ where $D_{ii} \sim U[0,1]$ and $V_{ij} \sim \mathcal{N}(0, 1/k)$. We then generate $X \in \mathbb{R}^{p \times n}$ where the $i^\mathrm{th}$ column of $X$ is generated according to $x_i \sim \mathcal{N}(\textbf{0}_p, \Sigma)$. We then set $y = X^\top \beta + \epsilon$ where $\epsilon_i \sim \mathcal{N}(0,1)$ and $\beta$ has a fixed number of nonzero regression coefficients each having equal magnitudes and random signs. We then estimate a factor model from the empirical covariance (see Section \ref{s:covest} in Appendix \ref{appendix_alg}) with rank equal to $k$, solve the appropriate SDP, sample the knockoffs 100 times and finally compare the FDR and power of the various methods in Figure \ref{fig:synth1}. The target FDR rate is set to 10\%. The results of Figure \ref{fig:synth1} confirm the fundamental trade off between maximizing power and minimizing FDR -- if the FDR is very low, we do not expect the method to have much power. However, since the knockoff procedure simply provides a bound on the FDR, we are interested in comparing which procedure provides the most power. We observe in Figure~\ref{fig:synth1} that the approximate solutions produced using Equi and ASDP constructions tend to be more conservative in their FDR control (which is well below the 10\% target) and often have significantly less power than the optimal full and low rank SDP solutions. Overall, these optimal SDP solutions have an empirical FDR closer to the target (sometimes marginally above due to model estimation error) and exhibit more power, probably because the knockoffs are less correlated. Surprisingly, the low rank solutions have more power than the full rank ones even when their FDR match, which might be explained by the implicit regularization effect of the low rank structure.

\begin{figure}[h!]
  \centering
  \begin{minipage}[b]{0.49\textwidth}
    \includegraphics[width=\textwidth]{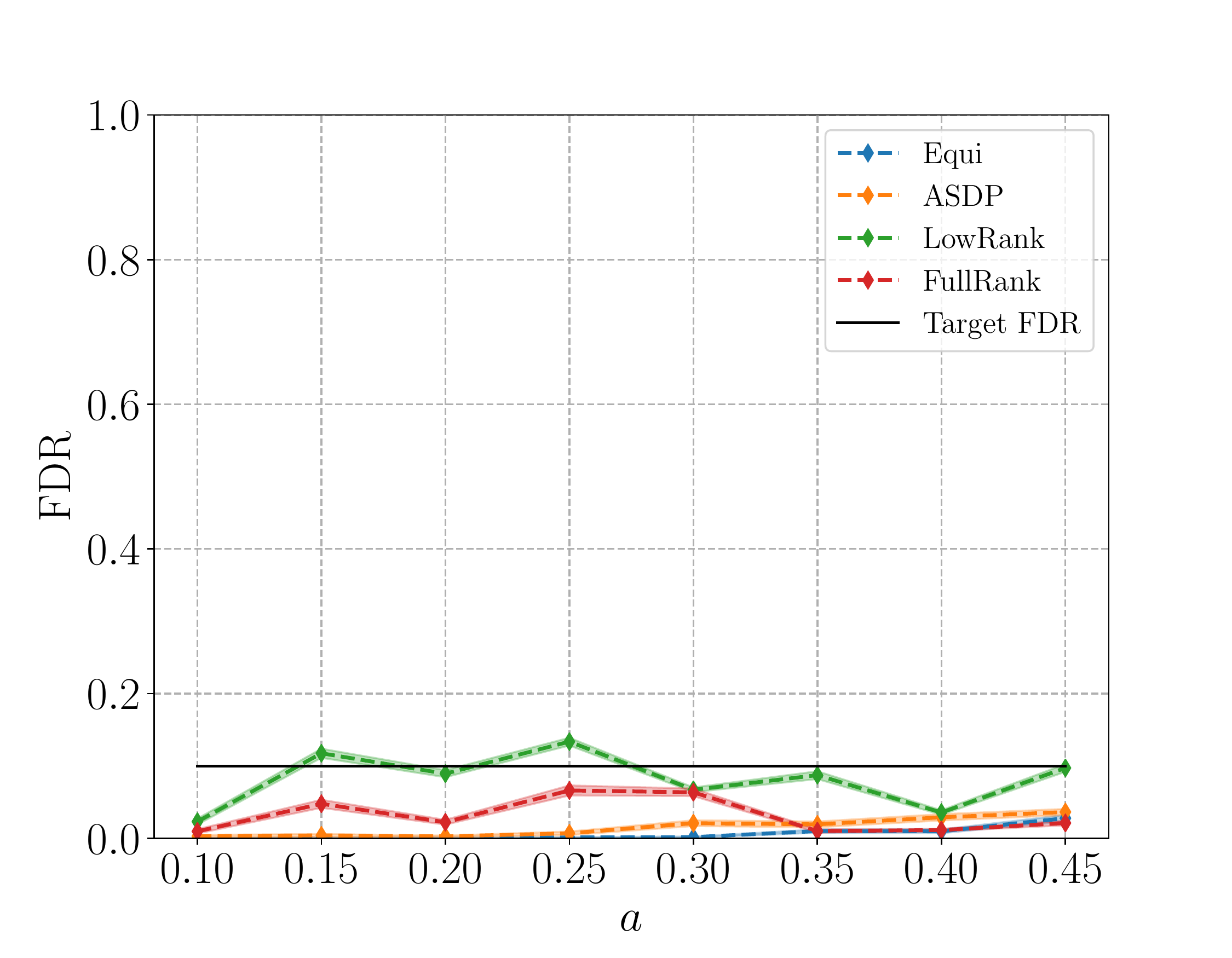}
  \end{minipage}
  \hfill
  \begin{minipage}[b]{0.49\textwidth}
    \includegraphics[width=\textwidth]{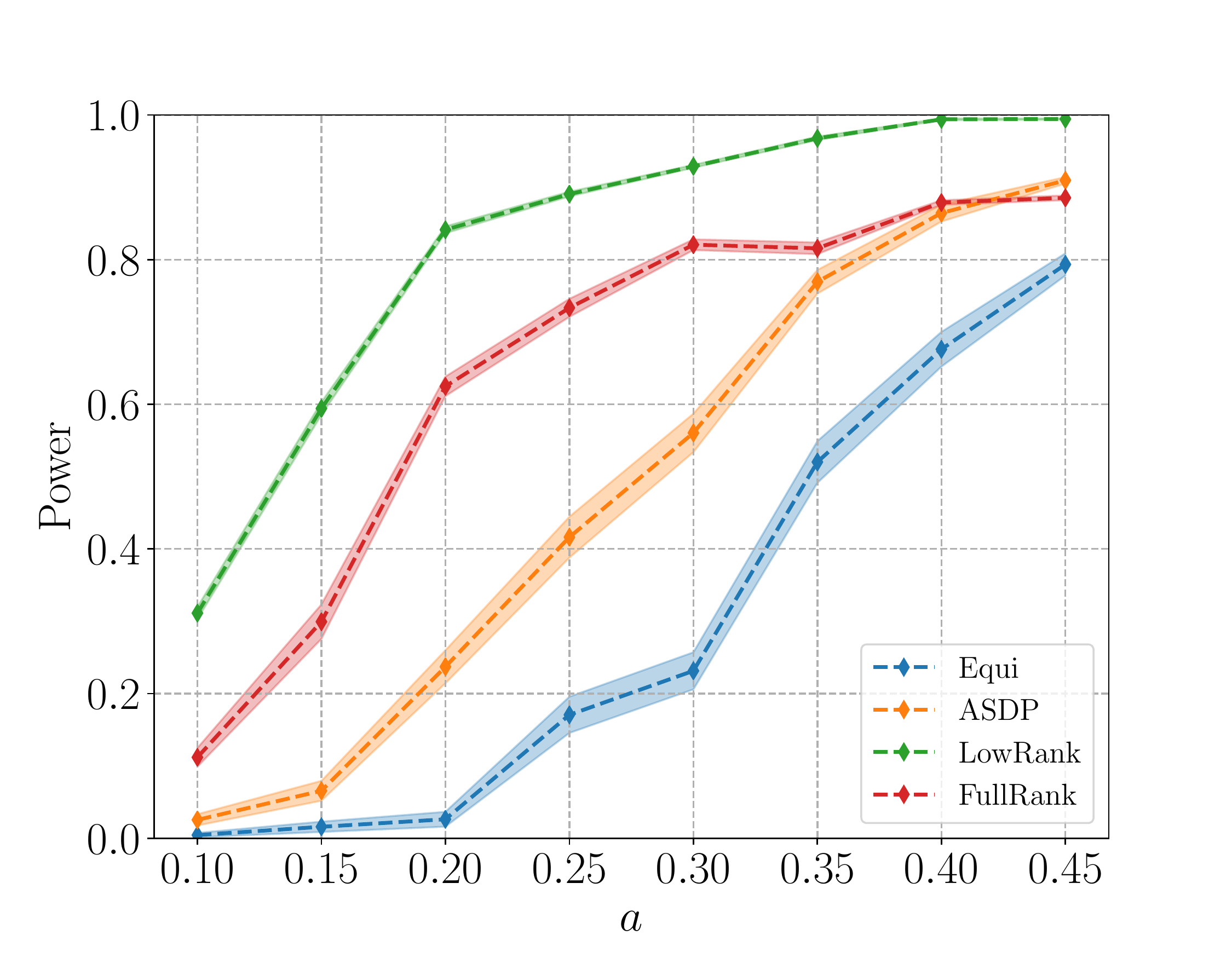}
  \end{minipage}
  \caption{$(n,p,k) = (1000,500,50)$, $\|\beta\|_0 = 50$ and each entry has equal amplitude. Each point represents 100 trials (the same $X$ and $\beta$ is used for each amplitude; the randomness is over the knockoff sampling) (\textit{Left}) FDP versus amplitude. (\textit{Right}) Power versus amplitude. }
  \label{fig:synth1}
\end{figure}

\subsection{fMRI feature selection}

We now test the low-rank factor model on the Human Connectome Project (HCP)~\citep{hcp_data} dataset for feature selection. Composed of brain connectivity maps, the dataset contains brain activity from $1,496$ healthy patients that was measured while they were shown pictures containing either humans faces or geometric shapes. We derive a binary classification task from the fMRI data consisting of identifying which pictures were shown to each patient given their brain activity. More specifically, we apply the knockoff filter to find which regions of the brain are the most discriminative for classification. Since fMRI data is by nature very noisy and extremely high dimensional,
we first perform a spacial clustering step resulting in $p = 5,000$ components. The factor model is then computed for the shrunk (Ledoit-Wolf) covariance matrix (see Section~\ref{s:covest}) with $k = 50$. Estimating the factor model, solving \eqref{eq:sdp_log_barrier}, sampling knockoffs and computing the covariates statistics takes roughly 20 seconds.
In this experiment, we make use of statistics derived from sparse centroids classifiers~\citep{centroids} which we found to be more effective than the LCD statistic~\citep{Cand18} for this classification task (see Section~\ref{ss:fmri_details} for further details). Figure~\ref{fig:fmri} shows the brain regions that were selected with a FDR target of $10\%$.
\begin{figure}[h!]
  \centering
  \begin{minipage}[b]{0.65\textwidth}
    \includegraphics[width=\textwidth]{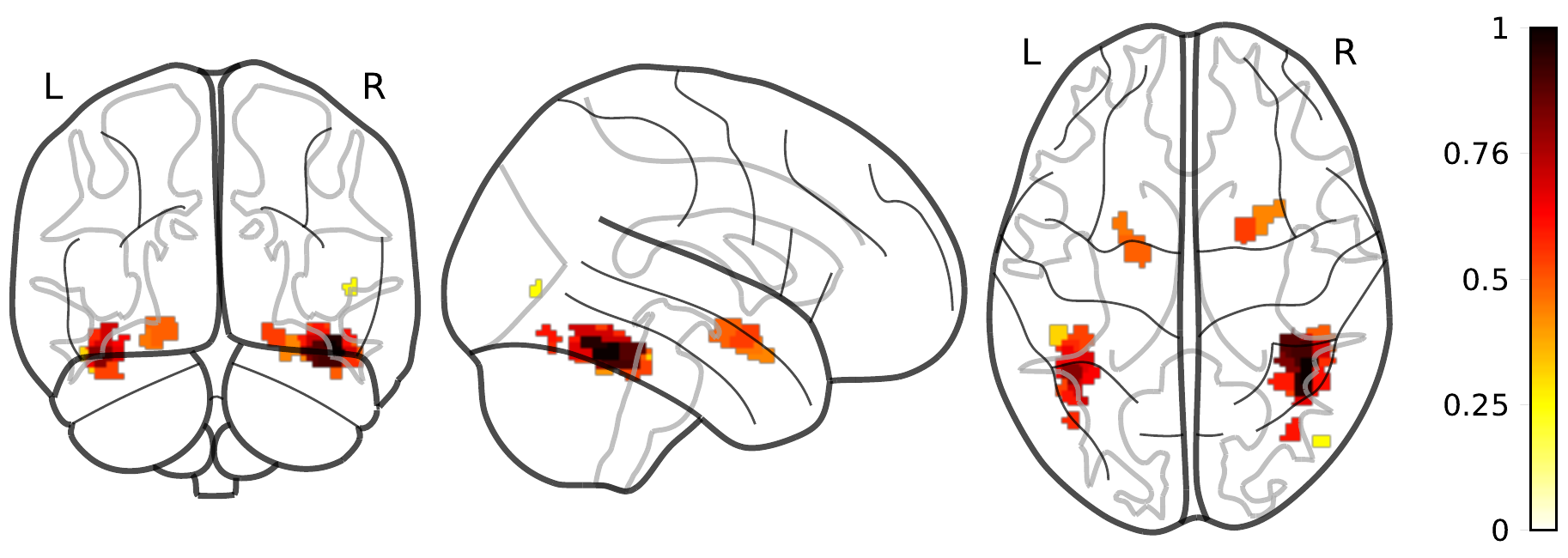}
  \end{minipage}
  \caption{
      Discoveries (34 in total) and their weights obtained by applying the knockoff filter with the low-rank factor model. In comparison, Equi-knockoffs did not result in any discoveries.
  }
  \label{fig:fmri}
\end{figure}
We cannot evaluate power or FDR here since the ground truth is not known. Note however that the discoveries are quite symmetric and concentrated in a few locations. Since the results were obtained without combining knockoffs with any additional structured penalty constraint to enforce localization or symmetry, this suggests that the features are indeed meaningful.

\section{Conclusion}
In this paper, we propose a computationally efficient method for computing Gaussian \textit{model-X} knockoffs.  For generic covariance matrices, our method scales as $\cO(p^3)$ and when we have a factor model assumption on the covariance matrix we are able to reduce the complexity down to $\cO(pk^2)$. We also provide computationally efficient methods for performing a factor model decomposition as well as sampling knockoffs. We validate our complexities empirically, compare the power/FDR of different knockoff generation methods on synthetic data, and qualitatively show the features selected by our procedure on fMRI data.


\section*{Acknowledgements}
We would like to thank the PARIETAL team (Inria-CEA) for sharing their fMRI data and helping preprocessing it. A.A. is at the d\'epartement d'informatique de l'ENS, \'Ecole normale sup\'erieure, UMR CNRS 8548, PSL Research University, 75005 Paris, France, and INRIA Sierra project-team. AA would like to acknowledge support from the {\em ML and Optimisation} joint research initiative with the {\em fonds AXA pour la recherche} and Kamet Ventures, a Google focused award, as well as funding by the French government under management of Agence Nationale de la Recherche as part of the "Investissements d'avenir" program, reference ANR-19-P3IA-0001 (PRAIRIE 3IA Institute).

{\bibliographystyle{plainnat}
\bibliography{ref.bib,MainPerso}}

\newpage 
\appendix

\section{Supplementary Material}\label{appendix_alg}

\subsection{Algorithms}
\subsubsection{Solving the SDP}\label{ss:algs}

We now fully spell out the various algorithms described in the text. The algorithms described in Sections \ref{sss:stable} and \ref{ss:sdp_low_rank} are detailed in Algorithm \eqref{alg:coordinate_ascent_log_barrier2} and Algorithm \eqref{alg:low_rank_coordinate_ascent} respectively.

\begin{algorithm}[H] 
    \caption{Stable coordinate ascent}\label{alg:coordinate_ascent_log_barrier2}
    \begin{algorithmic}[1]
        \STATE \textbf{Input:} $\Sigma$, barrier coefficient $\lambda > 0$,
        decay $\mu < 1$, $\s^{(0)} = \0_p$
        \STATE $\s = \s^{(0)}$
        \STATE $L = \call{Cholesky}{2 \Sigma - \diag\s}$
        \REPEAT
        \FOR{$j = 1,\,\dots,\, p$}
        \STATE Construct $\tilde{y}$ with $\tilde{y}_j = 0$ and $\tilde{y}_{j^c} = 2\Sigma_{j^c,j}$
        \STATE Solve $L x = \tilde{y}$
        \STATE $\zeta = 2 \Sigma_{j, j} - s_j$
        \STATE $c = \frac{\xi \|x\|_2^2}{\xi + \|x\|_2^2}$
        \STATE $s_j = \min\big(1,\max\big( 2\Sigma_{j,\, j} - c - \lambda,\, 0 \big)\big)$
        \STATE $\call{CholeskyUpdate}{L}$
        \ENDFOR
        \STATE $\lambda = \mu \lambda$
        \UNTIL{stopping criteria}
    \end{algorithmic}
\end{algorithm}

\begin{algorithm}[H]
    \caption{Coordinate Ascent using Factor Model}\label{alg:low_rank_coordinate_ascent}
    \begin{algorithmic}[1]
        \STATE \textbf{Input:} approximation $\cove = D + U U^\top$, barrier coefficient $\lambda > 0$,
        decay $\mu < 1$, $s^{(0)} = 0_p$
        \STATE $s = s^{(0)}$
        \STATE $M = U^\top \left( 2D - \diag\s \right)^{-1} U$
        \STATE $Q, R \leftarrow \call{DecomposeQR}{I_k + 2M}$
        \REPEAT
        \FOR{$j = 1,\,\dots,\, p$}
        \STATE $z \leftarrow \text{$j$th column of $U^\top$}$
        \STATE $\kappa \leftarrow \left( s_j - 2 D_{j,j} \right)^{-1}$
        \STATE $\call{UpdateQR}{Q, R, 2 \kappa, z}$
        \STATE $y \leftarrow \frac{Q R z - z}{2}$
        \STATE Solve $R x = Q^\top y$
        
        \STATE  $\alpha^\star = 2 \Sigma_{j,j} - 4z^\top y - \lambda + 8 y^\top x$
        \STATE $\s_j \leftarrow \min(1,\max(\alpha^\star,0))$
        \STATE $\kappa \leftarrow \left( 2 D_{j,j} - s_j \right)^{-1}$
        \STATE $\call{UpdateQR}{Q, R, 2 \kappa, z}$
        \ENDFOR
        \STATE $\lambda = \mu  \lambda$
        \UNTIL{stopping criteria}
    \end{algorithmic}
\end{algorithm}

\subsubsection{Stable updates (continued)} \label{s:stable_updates_cont}
Here, we outline the details of Algorithm \eqref{alg:coordinate_ascent_log_barrier2} which is an efficient version of Algorithm~\ref{alg:coordinate_ascent_log_barrier}. The key idea is to keep a Cholesky factorization of $2\Sigma - \diag s$ at any time. For a given index $j\in\{1,\hdots,p\}$ let $\tilde{y}_i\in\reals^p$, with
\begin{align*}
    \tilde{y}_i = \begin{cases}
    2\Sigma_{i,j} & \text{if } i\not = j \\
    0 & \text{otherwise}
    \end{cases}
\end{align*}
that is $\tilde{y}$ is the $j$th column of $2\Sigma$ with the $j^\mathrm{th}$ entry set to zero. Furthermore, let $x$ be the solution of the system $Lx = \tilde{y}$, we claim that we can compute $s^\star_j$ in the optimality condition~\eqref{eq:opt_cond} from $x$, with
\begin{align}\label{eq:quadratic_equality}
    4\Sigma_{j^c,j}^\top Q_j^{-1} \Sigma_{j^c,j} = \dfrac{\zeta \|x\|_2^2}{\zeta + \|x\|_2^2}
\end{align}
where $\zeta = 2\Sigma_{j,j} - s_j$. Note that computing $x$ given $L$ amounts to forward substitution and requires $\cO(p^2)$ steps. To prove \eqref{eq:quadratic_equality}, we can assume, up to a permutation and without loss of generality that $j=p$, so that
\begin{align}\label{eq:cholesky_system}
    \|x\|_2^2 &= \tilde{y}^\top (LL^\top)^{-1} \tilde{y} 
    = \begin{bmatrix}
    2\Sigma_{j^c,j} \\
    0
    \end{bmatrix}^\top A^{-1} \begin{bmatrix}
    2\Sigma_{j^c,j} \\
    0
    \end{bmatrix}  
\end{align}
where 
\begin{align*}
    A &= \begin{bmatrix}
    Q_j & 2\Sigma_{j^c,j}^\top\\
    2\Sigma_{j^c,j} & \zeta
    \end{bmatrix}, \;\;\;\;\;
    A^{-1} = \begin{bmatrix}
        B & * \\
        * & *
    \end{bmatrix}
\end{align*}
and inverse of $A$ has the block structure given above where $B = Q_j^{-1} + \frac{4}{\beta} Q_j^{-1} (\Sigma_{j^c,j} \Sigma_{j^c,j}^\top) Q_j^{-1}$ and $\beta = \zeta - 4\Sigma_{j^c,j}^\top Q_j^{-1} \Sigma_{j^c,j}$. Plugging this into \eqref{eq:cholesky_system} and simplifying, we arrive at
\begin{align*}
    \|x\|_2^2 = 4\Sigma_{j^c,j}^\top Q_j^{-1} \Sigma_{j^c,j} + \dfrac{4(2\Sigma_{j^c,j}^\top Q_j^{-1} \Sigma_{j^c,j})^2}{\zeta - 4\Sigma_{j^c,j}^\top Q_j^{-1} \Sigma_{j^c,j}}
\end{align*}
which yields \eqref{eq:quadratic_equality}. After computing $s_j^\star$, we perform a rank one Cholesky update of $L$ to maintain the equality $LL^\top = 2\Sigma - \diag s$.

\subsubsection{Coordinate Ascent under Factor Model (continued)} \label{ss:ascent_factor_model_cont}


The factor model assumption allows us to now solve a $k \times k$ linear system as opposed to a $(p - 1) \times (p - 1)$ linear system.
Remember that the update is written as $s_j \leftarrow \min(1,\max(\alpha^\star,\, 0))$ where
\BEQ\label{eq:decomp2}
    \alpha^\star = \underbrace{2 \Sigma_{j,j} - 4 U_{j,:} M_j U_{j,:}^\top - \lambda}_\text{$(*)$} ~ + ~ \underbrace{8 U_{j,:} M_j (I_k + 2M_j)^{-1} M_j U_{j,:}^\top}_\text{$(**)$}
\EEQ
We must be efficient in computing $M_j$ since directly forming it costs  $\cO(p  k^2)$ operations. To this end, let 
\begin{equation*}
M = U^\top \left( 2D - \diag(s) \right)^{-1} U
\end{equation*}
and notice that 
$M_j = M - (2D_{j,j} - s_j)^{-1} U_{j,:}^\top U_{j,:}$ is a rank one update of $M$, while an update of the $j^\mathrm{th}$ coordinate of $s$ is also a rank one update of $M$.
This means that we can efficiently compute $\alpha^\star$ by performing successive rank one updates on $k \times k$ matrices at each iteration. Indeed, suppose that we have a $QR$ decomposition of $\idm_k + 2M$. Using a rank one update of complexity $\cO\left( k^2 \right)$, we can get the following decomposition
\begin{equation*}
    Q^\prime R^\prime = \idm_k + 2M_j
\end{equation*}
From these factors, we get $2(Q^\prime R^\prime U_{j,:}^\top - U_{j,:}^\top) = 4 M_j U_{j,:}^\top$, hence the term $(*)$ in \eqref{eq:decomp2}.

The term $(**)$ involves computing the inverse of $\idm_k + 2M_j$. Using the $Q'R'$ factorization again, we solve for $x$ in the triangular system $R^\prime x = {Q^\prime}^\top M_j U_{j,:}^\top$, then form $U_{j,:} M_j x$. Finally, after $s_j$ has been updated, we perform a rank one update  on the $QR$ decomposition of $\idm_k + 2M$. The algorithm taking advantage of the factor model structure is summarized in Algorithm \ref{alg:low_rank_coordinate_ascent}.

\subsection{Sampling knockoffs (continued)}

In this section, we detail the efficient knockoffs sampling algorithms mentioned in Section~\ref{s:sampling}.

\paragraph{Forming $B$ and $\Delta$.} Given a diagonal plus low-rank covariance $\Omega = C + Z Z^\top$,
Algorithm~\ref{alg:form_B} (from~\cite{smola2004ldl}) forms the matrices $B \in \R^{p \times k}$ and $\Delta \in \R^{p \times p}$ such that $\Omega = L\left(Z,B\right)\; \Delta \; L\left(Z,B\right)^\top$.
It requires $\cO(pk)$ steps and $\cO(pk)$ additional memory (if only the diagonal of $\Delta$ is stored).
\begin{algorithm}[H]
    \caption{Forming the Cholesky factorization matrices $B$ and $\Delta$}\label{alg:form_B}
    \begin{algorithmic}[1]
        \STATE \textbf{Input:} $\Omega = C + Z Z^\top$
        \STATE $M = \idm_k,\, B = 0$
        \FOR{$j = 1,\,\dots,\, p$}
        \STATE $t = M z_j$
        \STATE $\Delta_{j,\,j} = C_{j,\, j} + z_j^\top t$ \COMMENT{always non-negative}
        \IF{$\Delta_{j,\,j} > 0$}
            \STATE $b_j = t / \Delta_{j,\,j}$
            \STATE $M = M - t t^\top / \Delta_{j,\,j}$
        \ELSE
            \STATE $b_j = 0$ \COMMENT{$b_j$ may be anything, choose 0 for simplicity}
        \ENDIF
        \ENDFOR
        \STATE \textbf{Output: } $\Delta$ and $B$ as in \eqref{eq:omega_ldlt}.
    \end{algorithmic}
\end{algorithm}

\paragraph{Fast multiplication.} Next, given the matrices $B, \Delta$ and a vector $v \in \R^p$, Algorithm~\ref{alg:sampling_cholesky} computes the product $u = L\left(Z,B\right) \Delta v$ in only $\cO(pk)$ operations (instead of the $\cO\left( p^2 \right)$ normally required for a matrix-vector product) and $\cO(p + k)$ memory.
\begin{algorithm}[H]
    \caption{Fast Cholesky multiplication}\label{alg:sampling_cholesky}
    \begin{algorithmic}[1]
        \STATE \textbf{Input:} $B,\Delta$ such that $\Omega = L\left(Z,B\right)\; \Delta \; L\left(Z,B\right)^\top$ and a vector $v \in \R^{p}$
        \STATE $w = \0_k$
        \FOR{$j = 1,\,\dots,\, p$}
            \STATE $u_j = \sqrt{\Delta_{j,\,j}} v_j + z_j^\top w$
            \STATE $w = w + \sqrt{\Delta_{j,\, j}} v_j b_j$
        \ENDFOR
        \STATE \textbf{Output: $u = L\left(Z,B\right) \Delta v$}
    \end{algorithmic}
\end{algorithm}
A low asymptotic complexity is possible thanks to the special structure of $L\left(Z,B\right)$. More precisely note that for any $j \in [p]$
\begin{align*}
    u_j
    &= \big( L\left( Z,\, B \right) \sqrt{\Delta} v \big)_j\\
    &= \sqrt{\Delta_{j,\, j}} v_j + \sum_{i = 1}^{j - 1} z_j^\top b_i \sqrt{\Delta_{i,\, i}}v_i\\ 
    &= \sqrt{\Delta_{j,\, j}} v_j + z_j^\top w_j
\end{align*}
where $w_j = \sum_{i = 1}^{j - 1} b_i \sqrt{\Delta_{i,\, i}}v_i$. The buffer vector $w$ may be updated iteratively which allows to compute $u$ at low cost.

\paragraph{Sampling knockoffs.} We combine Algorithms~\ref{alg:form_B} and~\ref{alg:sampling_cholesky} in order to sample knockoffs. From Algorithm \ref{alg:sampling_cholesky}, it is clear that neither $B$, $\Delta$ nor $L\left(Z,B\right)$ need to be fully computed and stored in memory.
Instead, the rows of $B$ and the diagonal of $\Delta$ may be computed iteratively,
as shown in Algorithm~\ref{alg:merged_sampling},
which has a time complexity of $\cO\left(pk^2\right)$ and uses $\cO\left(k^2\right)$ memory.
Here a single value is sampled from $\cN(0, \Omega)$; it may be easily extended to sample the $n$ required knockoffs.
\begin{algorithm}[H]
    \caption{Fast Gaussian sampling}\label{alg:merged_sampling}
    \begin{algorithmic}[1]
        \STATE \textbf{Input:} $\Omega = C + ZZ^\top$ and $v \in \R^{p}$ a sample from $\cN\left( 0,\, \idm_p \right)$
        \STATE $M = I_k$
        \STATE $w = \0_k$
        \FOR{$j = 1,\,\dots,\, p$}
            \STATE $t = M z_j$
            \STATE $\delta_j = C_{j,\, j} + z_j^\top t$ \COMMENT{$\delta_j$ is always non-negative}
            \STATE $u_j = \sqrt{\delta_j} v_j + z_j^\top w$
            \IF{$\delta_j > 0$}
            \STATE $b_j' = t / \delta_j$ \COMMENT{$j$th row of $B$}
            \STATE $M = M - t t^\top / \delta_j$
            \STATE $w = w + \sqrt{\delta_j} v_j w_j$
            \ENDIF
        \ENDFOR
        \STATE \textbf{Output:} $u \in \R^p$ sampled from $\cN(0, \Omega)$
    \end{algorithmic}
\end{algorithm}


\subsection{Spectrum of $\Omega$}\label{ss:hybrid}
The careful reader may notice that $s$ computed via Algorithm \ref{alg:low_rank_coordinate_ascent} (which by construction satisfies $\diag (s) \preceq D + UU^\top$) need not satisfy $\diag (s) \preceq 2\Sigma$. This in turn implies $\Omega$ is not PSD. In order to circumvent this problem, we propose two procedures: the \textit{hybrid} approach, and the \textit{low rank} approach. In the hybrid approach, after obtaining $\hat{s}$ from Algorithm \ref{alg:low_rank_coordinate_ascent}, as in \cite{Cand18}, we solve
\[
    \gamma^\ast = \arg\max_{\gamma} \; \gamma \; : \; \diag (\gamma \hat{s}) \preceq 2 \Sigma
\]
which is a minimum eigenvalue problem that can be solved efficiently via bisection over $\gamma$. This then ensure that $\Omega \succeq 0$ when $s = \gamma^\ast \hat{s}$. In the low rank approach, we do as detailed in the previous section; that is, we assume $\Sigma = D +UU^\top$ and sample our knockoffs accordingly. While not theoretically justified, we show in Section \ref{s:numres} how this model is able to outperform most of the other methods in terms of both speed and performance while still seemingly controlling FDR.

\subsection{Estimating Factor Models}\label{s:covest}
In this section, we explain how to efficiently compute a low rank factor model of a covariance matrix $\hat{\Sigma} = \tfrac{1}{n} X X^\top$ constructed from sample points $X \in \mathbb{R}^{n \times p}$. The factor model $\hat{\Sigma} = D + U U^\top$ is computed by the following non-convex optimization problem 
\begin{align}\label{eq:factor_model_opt}
    (D^\ast,U^\ast) = \arg\min_{D, U \in \mathbb{R}^{p \times k}} \; \left\{\|\hat{\Sigma} - D - UU^\top\|_F^2 \; : \; D \succeq 0 \; \text{diagonal}\right\}
\end{align}
where $k \ll p$ is a user-specified rank. Note that when $k = p$, $D^\ast = 0$ and $U = V \Lambda^{1/2}$ where $\hat{\Sigma} = V\Lambda V^\top$. While \eqref{eq:factor_model_opt} is non-convex, we use an alternating minimization scheme for solving it to (local) optimality. Given $UU^\top$, solving for $D$ is direct, we simply set $D_{ii} = \max(0,\hat{\Sigma}_{ii} - U_{ii}^2$). Now, given $D$, getting the optimal~$U$ reduces to projecting $\hat{\Sigma} - D$ onto the space of rank $k$ PSD matrices. The optimal $U$ is given by $U^\ast = V\Lambda^{1/2}$ where $V \in \mathbb{R}^{p \times k}$ are the top $k$ eigenvectors of $\hat{\Sigma} - D$ associated with the top $k$ eigenvalues and $\Lambda \in \mathbb{R}^{k \times k}$ is a diagonal matrix with $\Lambda_{ii} = \max(0,\lambda_i)$ for $i = 1,\hdots k$ (note $\hat{\Sigma} - D$ need not be PSD). However when $p$ is extremely large, we are interested in computing the top $k$ eigenvector, eigenvalue pairs \textit{without} explicitly constructing $\hat{\Sigma}$ for it may be too large to store in memory. We can do this by simply computing the top left singular vectors of $X$ as in e.g. \citep{yurtsever2017sketchy}.

In the setting where $n \ll p$, the empirical covariance tends to be far from the population covariance matrix and is ill-conditioned. To alleviate this, \cite{ledoit2000well} use Stein shrinkage to compute a better estimate of $\Sigma$. We use the regularized covariance (also known as the Ledoit-Wolfe estimator) 
\begin{align}
    \tilde{\Sigma} = (1-\delta) \hat{\Sigma} + \delta \mu I_p, \;\;\;\; \mu = \Tr(\Sigma)/p, \;\;\;\; \delta^\star = \frac{1}{n^2}\frac{\sum_{i = 1}^n (x_i^\top x_i)^2 - n\Tr(\Sigma^2)}
        {\Tr(\Sigma^2) - \Tr(\Sigma)^2 / p}
\end{align}
where $\delta^\star$ is the optimal shrinkage parameter.
These traces may be approximated with stochastic Lanczos quadrature~\citep{ubaru2017trace} without explicitly evaluating $\Sigma$
or $\Sigma^2$.

\section{Experimental Details}\label{appendix_exp}
\subsection{Benchmarks}\label{ss:benchmark_details}

The tolerances for the four methods were set to the following
\begin{enumerate}
    \item \texttt{CVXOPT}: Default 
    \item \texttt{SCS}: \texttt{eps = 1e-6}
    \item FullRank: \texttt{eps = 1e-8}
    \item LowRank: \texttt{eps = 1e-6}
\end{enumerate}
For \texttt{CVXOPT} and \texttt{SCS} the default settings were used. For the full rank and low rank models, our convergence criteria is the relative error on the objective value (i.e. $\tfrac{f(s_{k+1}) - f(s_k)}{f(s_k)} \leq p \cdot 10^{-6}$). \\

In addition to comparing the optimality of the methods based on objective functions, we check the feasibility of the solutions generated by the solutions. Figure \ref{fig:feas} plots the minimum eigenvalue of $2\Sigma - \diag (s)$ versus the dimension. If the minimum eigenvalue is negative, then the solution generated is infeasible. We see that with default tolerances, \texttt{CVXOPT} and \texttt{SCS} generate infeasible solutions whereas our models stay feasible. We noticed that decreasing the default tolerances of \texttt{CVXOPT} and \texttt{SCS} did not help much in this regard and significantly increased the run time of the methods. 
\begin{figure}[h] 
  \centering
    \includegraphics[scale=0.4]{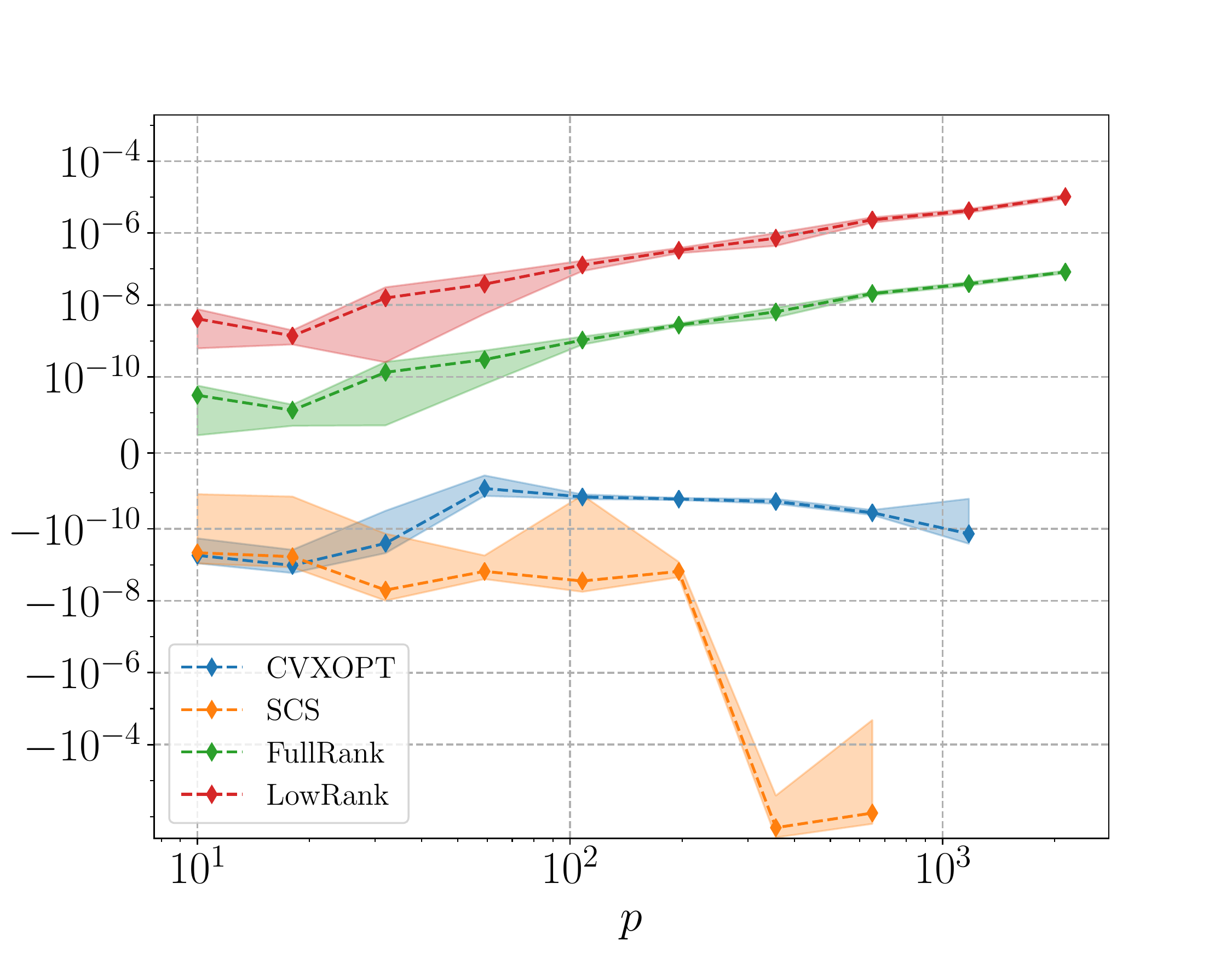}
  \caption{Feasibility plot of $\lambda_{\min}(2\Sigma - \diag (s))$ versus dimension.}
  \label{fig:feas}
\end{figure}
The drop in feasibility for \texttt{SCS} in Figure~\ref{fig:feas} is due to the fact that we reduced the tolerance threshold from $10^{-6}$ to $10^{-4}$ for the last two points because the convergence was extremely slow.

\subsection{Complexity}\label{ss:complexity_details}

\paragraph{SDP convergence.} Empirically, we observed that $5$ to $50$ cycles are enough to converge to a tolerance threshold of $10^{-6}$ on all the covariance matrices we experimented.

\paragraph{Sampling knockoffs.} Sampling from a multivariate normal distribution is traditionally done by finding the Cholesky decomposition of the covariance, which is our baseline.
In the case where the covariance is diagonal plus low-rank, we show that the knockoffs may be sampled in linear time.
However, our implementation of this algorithm is done in Python and \texttt{NumPy}.
We use Python loops because of the iterative nature of the algorithm.
This creates a lot of overhead and we expect the algorithm to be at least $10$ times faster if it were implemented in Cython.

\subsection{Synthetic Data}\label{ss:syntheticdata_details}
The error bars used to generate Figure \ref{fig:synth1} were divided by the square root of the number of trials in order to make it a $68\%$ confidence interval.

\subsection{fMRI (HCP) experiment} \label{ss:fmri_details}

\paragraph{Preprocessing.} Connectivity maps are volumes of size $91 \times 109 \times 91$.
Among these $902,629$ voxels only $212,445$ are in the brain envelope.
We first extract them because they contain the functional information of the brain.
Then, in order to average the noise and reduce the data dimension, we perform a spacial clustering step.
To do so, we make use of the package \texttt{Nilearn} which provides parcellation algorithms.
We employed the Ward clustering method~\citep{Johnson1967HierarchicalCS} because it is known to perform well in terms of accuracy~\citep{thirion2014clustering}.


\paragraph{Knockoffs statistics.} As the knockoff framework offers a lot of freedom regarding the choice of the covariates statistics, we chose to derive them from sparse centroid classifiers, primarily because it can be computed very efficiently as compared to the LCD statistic.
More specifically, for any $L_0$ penalty coefficient $\lambda \geq 0$ we define the sparse centroids parameters $\big( \hat \theta^+(\lambda),\; \hat\theta^-(\lambda) \big)$ as the solutions of the following optimization problem
\begin{align*}
    \big( \hat \theta^+(\lambda),\; \hat\theta^-(\lambda) \big) =  \underset{\theta^+, \theta^- \in \mathbb{R}^p}{\mathrm{arg\,min}} &\; \dfrac{1}{n_+} \sum_{j \in \mathcal{J}^+} \|x_j - \theta^+\|_2^2 + \dfrac{1}{n_-} \sum_{j \in \mathcal{J}^-} \|x_j - \theta^-\|_2^2 + \lambda \|\theta^+ - \theta^-\|_0
\end{align*}
where $\mathcal{J}^\pm$ denotes an index set corresponding to the $\pm 1$ labeled data points and $n_\pm = |\mathcal{J}^\pm|$.
Following the same idea as the LSM statistic~\citep{Cand18},
we define $Z_j = \sup \{\lambda \geq 0 \mid \hat \theta^+(\lambda) \neq \hat\theta^-(\lambda)\}$ for all $j \in [2p]$.
Finally, our statistic takes the following form
\begin{align}\label{eq:centroids_stats}
    W_j = |Z_j| - |Z_{j+p}|
\end{align}
using the difference function which is antisymmetric.
The knockoff filter controls the FDR only if the statistics obey the \emph{flip-sign} property as explained in Section 3.2 of \citep{Cand18}.
It is easy to verify that the statistics defined in Equation~\eqref{eq:centroids_stats} satisfy the requirements.

We also experimented LCD statistics on fMRI data. The computation takes roughly $5$ minutes (as opposed to 2 seconds for the centroids) and the procedure selects approximately the same regions and the same features.
Figure~\ref{fig:fmri_lcd} shows the features that were selected with a FDR target of $10\%$ using LCD statistic.
\begin{figure}[h!]
  \centering
  \begin{minipage}[b]{0.65\textwidth}
    \includegraphics[width=\textwidth]{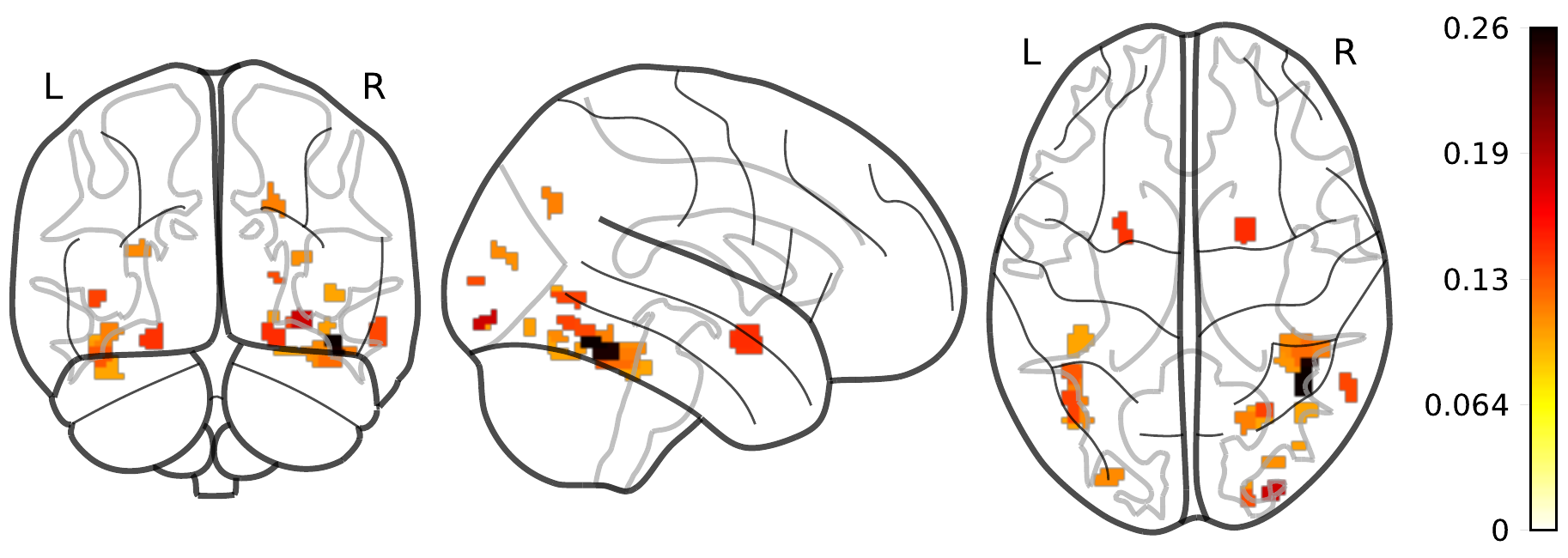}
  \end{minipage}
  \caption{
      Discoveries (26 in total) and their weights obtained by applying the knockoff filter with the low-rank factor model and LCD statistic.
  }
  \label{fig:fmri_lcd}
\end{figure}

\end{document}